\documentclass[12pt]{colt2023} 

\usepackage{wrapfig}

\title[Riemann Hypothesis and Neural Networks]{On the Connection Between Riemann Hypothesis\\ and a Special Class of Neural Networks}
\usepackage{times}
\usepackage[capitalize]{cleveref}

\coltauthor{%
 \Name{Soufiane Hayou} \Email{hayou@berkeley.edu}\\
 \addr Simons Institute\\
 UC Berkeley
}

\newcommand{\reals}{\mathbb{R}}

\newcommand{\realpart}{\textup{Re}}
\newcommand{\complex}{\mathbb{C}}

\newcommand{\nn}{\mathcal{N}}

\newcommand{\RH}{$\mathcal{RH}$}

\newtheorem{thm}{Theorem}

\begin{document}

\maketitle

\begin{abstract}
The Riemann hypothesis (\RH) is a long-standing open problem in mathematics. It conjectures that non-trivial zeros of the zeta function all lie on the line $\realpart(z) = 1/2$. The extent of the consequences of \RH~is far-reaching and touches a wide spectrum of topics including the distribution of prime numbers, the growth of arithmetic functions, the growth of Euler's totient, etc. In this note, we revisit and extend an old analytic criterion of the \RH~known as the Nyman-Beurling criterion which connects the \RH~to a minimization problem that involves a special class of neural networks. This note is intended for an audience unfamiliar with \RH. A gentle introduction to \RH~is provided.
\end{abstract}

\section{Introduction}
The Riemann hypothesis conjectures that the non-trivial zeros of the Riemann zeta function are located on the line $\realpart(z) = \frac{1}{2}$ in the complex plane $\complex$. This is a long-standing open problem in number theory first formulated by \citep{riemann}. The Riemann zeta function was first defined for complex numbers $z$ with a real part greater than 1 by $\zeta(z) = \sum_{n=1}^\infty \frac{1}{n^z}, z \in \mathbb{C}, \realpart(z)>1. $
However, it is the extension of the zeta function $\zeta$ to the whole complex plane $\complex$ that is considered in the statement of \RH. This extension is called the \emph{analytic continuation} of the zeta function (details are provided in \cref{app:rh_step_by_step}).\\
There is strong empirical evidence that \RH~holds. Recent numerical verification by \cite{platt_rh_true} showed that \RH~is at least true in the region $\{z = a + i b \in \complex: a\in (0,1),b \in (0, \gamma]\}$ where $\gamma = 3\cdot 10^{12}$, meaning that all zeros of the zeta function with imaginary parts in $(0,\gamma]$ have a real part equal to $\frac{1}{2}$. Several other theoretical insights seem to support \RH~;we invite the reader to check \cref{app:rh_step_by_step} for a short summary of relevant results and insights. In this note, we are interested in an specific criterion of the \RH, i.e. an equivalent statement of \RH. This criterion is known as the Nyman-Beurling criterion \citep{nyman, beurling} which states that \RH~holds if and only if a special class of functions is dense in $L_2(0,1)$. This class of functions can be seen as a special kind of neural networks with one dimensional input. In this note, we show that the sufficient condition can be easily extended to $L_2((0,1)^d)$. Specifically, we introduce a new class of neural networks and show that \RH~implies the density of this class in $L_2((0,1)^d)$ for any $d\geq 2$. The necessary condition in general dimension $d\geq 2$ remains an open question.

\section{Riemann Hypothesis}\label{sec:rh}
The Riemann zeta function was originally defined for complex numbers $z$ with a real part greater than $1$ by 
\begin{equation}\label{eq:zeta_function}
\zeta(z) = \sum_{n=1}^\infty \frac{1}{n^z}, \quad z \in \mathbb{C}, \realpart(z)>1. 
\end{equation}
The above definition of Riemann zeta function excludes the region of interest $\{z \in \complex: \realpart(z)=\frac{1}{2}\}$ since the series in \cref{eq:zeta_function} diverge when $|z|<1$. Indeed, \RH~is stated for the an extension of the zeta function on the whole complex plane $\complex$. This extension is called the analytic continuation, and it is unique by the Identity theorem \citep{walz_identity}. To give the reader some intuition of how such extension is defined, let us show how we can extend $\zeta$ to the region $\{z \in \complex: \realpart(z)>0\}$. Observe that the function $\zeta$ satisfies the following identity
$$
(1 - 2^{1-z})\zeta(z) = \sum_{n=1}^\infty \frac{1}{n^z} - 2 \sum_{n=1}^\infty \frac{1}{(2n)^z} =  \sum_{n=1}^\infty \frac{(-1)^{n+1}}{n^z},
$$
where the right hand side is defined for any complex number $z$ such that $\realpart(z)>0$. Using similar techniques, we can show that for any $z\in \complex$ such that $\realpart(z) \in (0,1)$,
\begin{equation}\label{eq:func_eq}
\zeta(z) = 2^z \pi^{z-1} \sin\left(\frac{\pi z}{2}\right) \Gamma(1-z) \zeta(1-z),
\end{equation}
which helps extend $\zeta$ to complex numbers with negative real part. A step by step explanation of the analytic continuation of the $\zeta$ function is provided in \cref{app:rh_step_by_step}.

\paragraph{Zeros of the $\zeta$ function.} From \cref{eq:func_eq}, we have $\zeta(-2k) = 0$ for any integer $k \geq 1$. The negative even integers $\{-2 k \}_{k \geq 1 }$ are thus called \emph{trivial zeros} of the Riemann zeta function since the result follows from the simple fact that $\sin\left( - \pi k \right) = 0$ for all integers $k \geq 1$. The other zeros of $\zeta$ are called non-trivial zeros, and their properties remain poorly understood. The \RH~conjectures that they all lie on a the line $\realpart(z)=\frac{1}{2}$.

\paragraph{Riemann Hypothesis (\RH).} \emph{All non-trivial zeros of $\zeta$ have a real part equal to $\frac{1}{2}$}.\\
Whether \RH~holds is still an open question. The consequences of the Riemann hypothesis are various (see \cref{app:rh_step_by_step}) and numerous equivalent results exist in the literature. In the next section, we re-visit an old analytic criterion of \RH that involves a special type of functions that can be seen as single layer neural networks.

\subsection{A \emph{Neural Network} Criterion for \RH}
For $p>1, d \in \mathbb{N}\backslash\{0\}$, and some set $S \subset \reals^d$, let $L_p(S)$ denote the set of real-valued functions $f$ defined on $S$ such as $|f|^p$ is Lebesgue integrable, i.e. 
$
L_p(S) = \{ f:S \to \reals: \int_S |f|^p d\mu < \infty\},
$
where $\mu$ is the Lebesgue measure on $\reals^d$. We denote by $\|.\|_p$ the standard Lebesgue norm defined by $\|f\|_p = \left( \int_S |f|^p d\mu\right)^{1/p}$ for $f \in L_p(S)$. 

\begin{figure}
    \centering
    \includegraphics[width=0.9\linewidth]{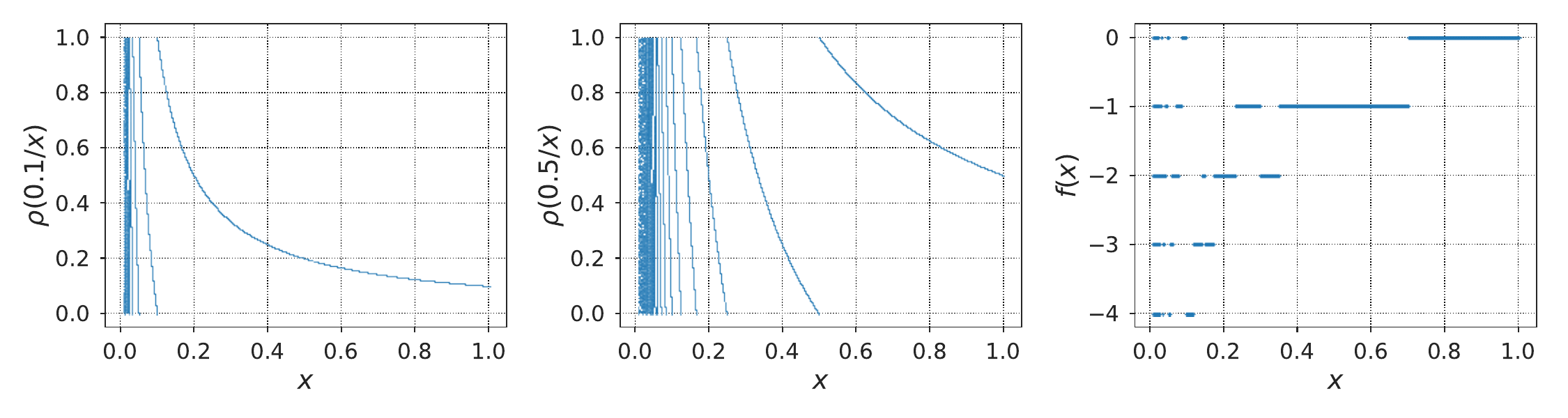}
    \caption{\small{(\textbf{Left}) The curve of the function $x \to \rho(\alpha / x)$ for $\alpha \in \{0,1; 0.5\}$. (\textbf{Right}) The graph of the function $f(x) = \rho(0.7/x) - \rho(0.3/x) - 4 \rho(0.1/x)$, which belongs to the class $\nn$.}}
    \label{fig:rho_plots}
\end{figure}
For some $k \geq 1$, let $I_k \overset{def}{=} (0,1)^k = (0,1)\times \dots \times (0,1)$ where the product contains $k$ terms.  Let $\rho$ denote the fractional part function given by
$
\rho(x) = x - \lfloor x \rfloor,$ for $x \in \reals$. Consider the following class of functions defined on the interval $I_1$
$$
\nn = \{ f(x) = \sum_{i=1}^m c_i \rho\left(\frac{\beta_i}{x}\right), x \in I_1: m \geq 1, c \in \reals^m, \beta \in I_m, c^T\beta = 0 \}.
$$
In machine learning nomenclature, $\nn$ consists of single-layer neural networks with a constrained parameter space and a specific non-linearity (or activation function) that depends on the fractional part $\rho$. The parameters $(c,\beta)$ belong to the set $\{c \in \reals^m, \beta \in (0,1)^m, c^T\beta = 0\}$. The values $(\rho(\beta_i/x))_{1\leq i \leq m}$ act as the neurons (post-activations) in the neural network. In \cref{fig:rho_plots}, we depict neuron values for different choices of $\beta_i$. The graphs show fluctuations when $x$ is close to $0$ which should be expected since the function $x\to \rho(\beta_i/x)$ fluctuates indefinitely between $0$ and $1$ as $x$ goes to zero, whenever $\beta_i \neq 0$. In figure \cref{fig:rho_plots} (right), we show an example of a function from the class $\nn$ given by $f(x) = \rho(0.7/x) - \rho(0.3/x) - 4 \rho(0.1/x)$. We observe that $f$ is a step function which might be surprising at first glance. However, it is easy to see that $\nn$ consists only of step functions. This is due to the constraint on the parameters $c, \beta$, and the fact that $\rho(x) = x - \lfloor x \rfloor$.
Now, we are ready to state the main results that draw an interesting connection between \RH~and the class $\nn$.
\begin{thm}[\cite{nyman}]\label{thm:nyman}
The \RH~is true if and only if $\nn$ is dense in $L_2(I_1)$.
\end{thm}
\cite{beurling} later extended this result by showing that for any $p>1$, the $\zeta$ function has no zeroes in the set $\{ z \in \complex: \realpart(z)>1/p\}$ if and only if the set $\nn$ is dense in $L_p(I_1)$.

\begin{thm}[\cite{beurling}]\label{thm:beurling}
The Riemann zeta function is free from zeros in the half plane $Re(z)>\frac{1}{p}, 1<p<\infty$, if and only if $\nn$ is dense in $L_p(I_1)$.
\end{thm}
The intuition behind this connection is rather simple. The number of fluctuations of the function $x \to \rho(\beta/x)$ near $0$ is closely related to the $\zeta$ function. To understand the machinery of the proofs of \cref{thm:nyman} and \cref{thm:beurling}, we provide a sketch of the proof by \cite{beurling} for the sufficient condition in \cref{app:nyman_beurling}. 
Using the same techniques, we derive the following result on zero-free regions of the zeta function.
\begin{lemma}[Nyman-Beurling zero-free regions]\label{lemma:zero_free_region_1}
Let $f \in \nn$ and $\delta = \|1 - f\|_2$ be the distance between the constant function $1$ on $I_1$ and $f$. Then, the region $\{z \in \complex, \realpart(z) > \frac{1}{2} \left(1 + \delta^2 |z|^2\right)\}$ is free of zeroes of the Riemann zeta function $\zeta$.
\end{lemma}
The condition that $\nn$ should be dense in $L_2(I_1)$ can be replaced by the following weaker condition: the constant function $1$ on $I_1$ can be approximated up to an arbitrary accuracy with functions from $\nn$. This is because from the constant function $1$, one can construct an approximation of any step-wise function, which in turn can approximate any function in $L_2(I_1)$.

A discussion on the empirical implications of \cref{lemma:zero_free_region_1} is provided in \cref{app:nyman_beurling}. In the next section, we show that the sufficient condition of \cref{thm:beurling} can be easily generalized to networks with multi-dimensional inputs, i.e. the case $d\geq 1$. 
\section{A sufficient condition in the multi-dimensional case}
Let $d \geq 1$ and consider the following class of neural networks with inputs in $I_d$,
$$
\nn_d = \{ f(x) = \sum_{j=1}^d \sum_{i=1}^m c_{i,j} \rho\left(\frac{\beta_{i,j}}{x_j}\right), x \in I_d: m \geq 1, c \in \reals^{d\times m}, \beta \in I_{d \times m}, c^T\beta = 0 \},
$$
where $c = (c_{1,1}, c_{2,1}, \dots, c_{m,1}, \dots, c_{m, d})^\top \in \reals^{m d}$ is the flattened vector of $(c_{.,j})_{1 \leq j \leq d}$. Notice that we recover the Nyman-Beurling class $\nn$ when $d=1$. Using this class, we can generalize the zero-free region result given by \cref{lemma:zero_free_region_1} to a multi-dimensional setting in the case $p=2$.\footnote{The choice of $p=2$ is arbitray, and similar result to that of \cref{lemma:zero_free_region_d} can be obtained for any $p>1$.}

\begin{lemma}[zero-free regions for general $d \geq 1$]\label{lemma:zero_free_region_d}
Let $d\geq 1$ and $f \in \nn_d$. Let $\delta = \|1 - f\|_2$ be the $L_2$ distance between the constant function $1$ on $I_d$ and $f$. Then, the region $\{z \in \complex, \realpart(z) > \frac{1}{2} \left(1 + \delta^{\frac{2}{d}} |z|^2\right)\}$ is free of zeroes of the Riemann zeta function $\zeta$.
\end{lemma}

\begin{wrapfigure}{r}{0.4\textwidth}
\vspace{-3em}
  \begin{center}
    \includegraphics[width=0.4\textwidth]{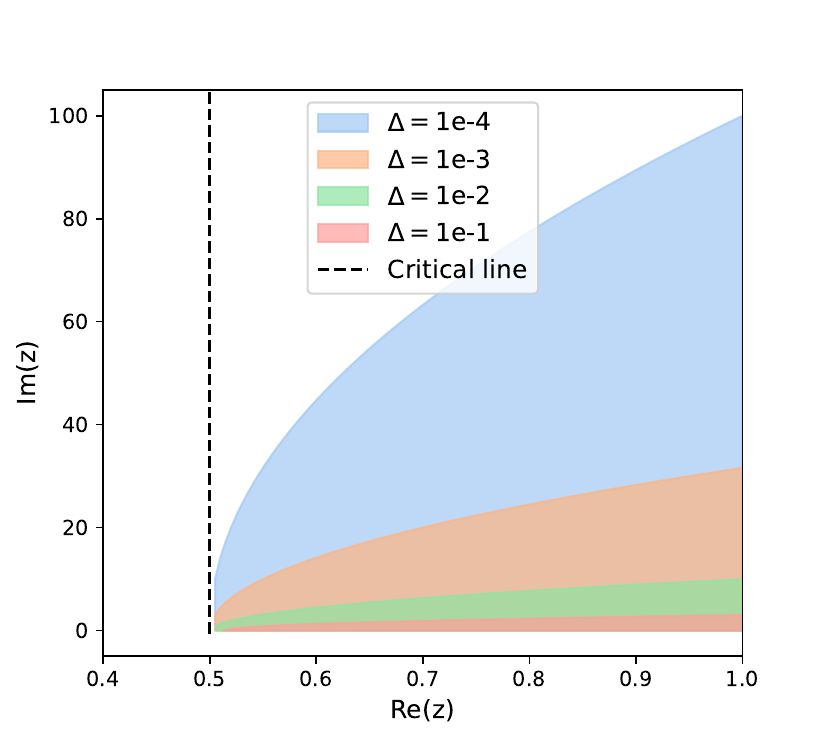}
  \end{center}
  \vspace{-1.5em}
  \caption{Zero-free regions of the form $\{\realpart(z) > \frac{1}{2}(1 + \Delta |z|^2)\}$ as stated in \cref{lemma:zero_free_region_1}, \cref{lemma:zero_free_region_d}, and \cref{lemma:zero_free_region_emp}.}\label{fig:zer_free_regions}
\end{wrapfigure}
In \cref{fig:zer_free_regions}, we depict the zero-free regions from Lemma \ref{lemma:zero_free_region_d}. The smaller the constant $\delta$, the larger the region. The multi-dimensional input case ($d \geq 2$) can therefore be interesting if we can better approximate the constant function $1$ with functions from $\nn_d$. More precisely, the result of Lemma \ref{lemma:zero_free_region_d} is relevant if for some $d\geq 2$, we could find $\delta$ such that $\delta^{2/d} < \delta_1$, where $\delta_1$ is the approximation error in the one-dimensional case $d=1$. In this case, the zero-free region obtained with $d\geq 2$ will be larger than the one obtained with $d=1$. We refer the reader to \cref{app:empirical_implications} for a more-in depth discussion about the empirical implications of the multi-dimensional case.
Notice that if $\delta$ can be chosen arbitrarily small, then the zero-free region in Lemma \ref{lemma:zero_free_region_d} can be extended to the whole half-plane $\{\realpart(z) > 1/2\}$. This is a generalization of the sufficient condition of \cref{thm:beurling} in the multi-dimensional case.
\begin{corollary}[Sufficient condition for $d\geq 1$]
Let $d \geq 1$. Assume that the class $\nn_d$ is dense in $L_2(I_d)$. Then, the region $\{\realpart(z) > 1/2\}$ is free of the zeroes of the Riemann zeta function $\zeta$.
\end{corollary}

\subsection{Open problem: The necessary condition for $d\geq 2$}

By considering the class $\nn_d$, we generalized the sufficient condition of Beurling's criterion in the multi-dimensional input case $d \geq 2$. However, it is unclear whether a similar necessary condition holds. Proving that \RH~implies the density of $\nn_d$ in $L_2(I_d)$ is challenging. A function $f \in \nn_d$ can be expressed as $f(x) = \sum_{i=1}^d f_i(x_i)$ for $x = (x_1, \dots, x_d)^\top \in I_d$, and $f_i$ are functions with one-dimensional inputs. This special additive form of functions from $\nn_d$ makes it harder to use arguments similar to the one-dimensional case (\cref{thm:beurling}) to prove density results.

\section{Discussion on the Implications and Limitations}\label{app:empirical_implications}
In this section, we discuss some empirical implications of Lemma \ref{lemma:zero_free_region_1} and Lemma \ref{lemma:zero_free_region_d}.

\paragraph{Probabilistic zero-free regions.} Notice that Lemmas \ref{lemma:zero_free_region_1} and \ref{lemma:zero_free_region_d} require access to the distance $\|1 - f\|_2$ which is generally intractable. However, we can approximate this quantity using Monte Carlo samples and obtain high probability bounds for this norm. Hence, the best we can do with such criterion is to verify the non-existence of zeroes of $\zeta$ in some region \emph{with high probability}. Indeed, using Hoefdding's inequality, we have the following result.

\begin{lemma}\label{lemma:zero_free_region_emp}
Let $d\geq 1$, $N \geq 1$ and $X_1, X_2, \dots, X_N$ be iid uniform random variables on $I_d$. Let $f \in \nn_d$ (where for $d=1$, we denote $\nn_d = \nn$) such that $f(x) = \sum_{j=1}^d\sum_{i=1}^m c_{i,j} \rho\left( \frac{\beta_{i,j}}{x_j} \right)$ for all $x \in I_d$, for some $m \geq 1, \beta \in I_{m\times d}, c \in \reals^{m \times d}$. Then, for any $\alpha \in (0,1)$, we have with probability at least $1- \alpha$, the region $R_N \overset{def}{=} \{ \realpart(z) > \frac{1}{2}\left(1  + \Delta_N(f)^{1/d} |z|^2\right)\}$ is free of the zeroes of $\zeta$, where $\Delta_N(f) = \frac{1}{N} \sum_{i=1}^N (1 - f(x_i))^2 \, + (1 + \|c\|_1^2) \sqrt{\frac{2 \log(2/\alpha)}{N}}$, with $\|c\|_1 = \sum_{i=1}^m |c_i|$.
\end{lemma}
\begin{proof}
The proof follows from a simple application of Hoeffding's concentration inequality to control the deviations of the empirical risk $N^{-1}\sum_{i=1}^N (1 - f(x_i))^2$. Hoeffding's lemma requires that the random variables $(1 - f(x_i))^2$ are bounded, which is straightforward  since $(1 - f(x_i))^2 \leq 2 (1 + f(x_i)^2) \leq 2 (1 + \|c\|^2_1)$ almost surely.
\end{proof}
The result of \cref{lemma:zero_free_region_emp} has an important implication on the choice of the sample size. Indeed, to have the coefficient $\delta_N^{1/d}$ of order $\epsilon$ with high probability, a necessary condition is that $N = \mathcal{O}(\epsilon^{2d})$.

\paragraph{When is the multi-dimensional variant better than the one-dimensional criterion?}
For some $d \geq 2$, it is straightforward that the multi-dimensional criterion given in Lemma \ref{lemma:zero_free_region_d} is better than the one given in Lemma \ref{lemma:zero_free_region_1} only if $\inf_{f \in \nn_d} \|1 - f\|_2^{1/d} < \inf_{f \in \nn} \|1 - f\|_2$. Under this condition, the zero-free region is larger with $d \geq 2$. For empirical verification of the \RH~and for same probability threshold $\alpha$, Lemma \ref{lemma:zero_free_region_emp} implies that the multi-dimensional setting is better than the one-dimensional counterpart whenever $\inf_{f \in \nn_d} \Delta_N(f)^{1/d} < \inf_{f \in \nn} \Delta_N(f)$. We discuss the feasibility of such conditions in the next paragraph.

\paragraph{What does it take to improve upon existing numerical verifications of \RH?} The high probability zero-free regions from Lemma \ref{lemma:zero_free_region_emp} are of the form $\{\realpart(z) > \frac{1}{2}(1 + \Delta |z|^2)\}$ for some constant $\Delta>0$. \\
Using a different analytical criterion of the \RH,  \cite{platt_rh_true} showed that the region $\{a + ib: a \in (0,1), b \in (0, \gamma], \gamma \approx 3\cdot10^{12}\}$ is free of the zeroes of $\zeta$. Hence, using Lemma \ref{lemma:zero_free_region_emp} to improve this result requires that the region $R_N \cap \{ a + ib, a \in (0,1), b \in (0, \gamma]\}$ contains complex numbers $z$ with imaginary part larger than order $10^{12}$. Let $z = a + ib \in \complex$. Having $z \in R_N$ implies that $b^2 < -a^2 + \Delta_N(f)^{-1/d} (2 a - 1)$. For the region of interest where $a \in (0,1)$, and assuming that $\Delta_N(f)$ is small enough, the right-hand side is of order $\Delta_N(f)^{-1/d} (2a -1)$ which is maximized for $a = 1$ and equal to $\Delta_N(f)^{-1/d}$. Thus, to improve upon existing work \citep{platt_rh_true} (at least with some high probability certificate), we need to have $\Delta_N(f)^{-1/d}$ of order $10^{12}$, which means that $\Delta_N(f)^{1/d}$ should be at least of order $10^{-12}$. This requires a the minimize of a the empirical risk $N^{-1} \sum_{i=1}^N (1 - f(x_i))^2$ with a minimum sample size of order $10^{24}$ which is unfeasible with the current compute resources.

\newpage

\bibliography{sample}

\begin{thebibliography}{12}
\providecommand{\natexlab}[1]{#1}
\providecommand{\url}[1]{\texttt{#1}}
\expandafter\ifx\csname urlstyle\endcsname\relax
  \providecommand{\doi}[1]{doi: #1}\else
  \providecommand{\doi}{doi: \begingroup \urlstyle{rm}\Url}\fi

\bibitem[Beurling(1955)]{beurling}
A.~Beurling.
\newblock {A closure problem related to the Riemann Zeta-function}.
\newblock \emph{Proceedings of the National Academy of Sciences of the United
  States of America}, 41 (5):\penalty0 312–314, 1955.

\bibitem[Borwein(1995)]{borwein1995etafunction}
P.~Borwein.
\newblock {An efficient algorithm for the Riemann Zeta function}.
\newblock \emph{CECM-95-043}, 1995.

\bibitem[Dudek(2015)]{dudek_prime2015}
A.W. Dudek.
\newblock {On the Riemann hypothesis and the difference between primes}.
\newblock \emph{International Journal of Number Theory}, 11\penalty0
  (03):\penalty0 771--778, 2015.

\bibitem[Edwards(1974)]{edwards_1974}
H.M. Edwards.
\newblock \emph{Riemann's Zeta Function}.
\newblock Pure and Applied Mathematics, A Series of Monographs and Textbooks,
  1974.

\bibitem[Hardy and Wright(1938)]{hardy_1938}
G.~H. Hardy and E.~M. Wright.
\newblock \emph{An Introduction to the Theory of Numbers}.
\newblock Oxford University Press, 1938.

\bibitem[Nyman(1950)]{nyman}
B.~Nyman.
\newblock \emph{{On the one-Dimensional translation group and semi-group in
  certain function spaces}}.
\newblock 1950.

\bibitem[Platt and Trudgian(2021)]{platt_rh_true}
D.~Platt and T.~Trudgian.
\newblock {The Riemann hypothesis is true up to 3·1012}.
\newblock \emph{Bulletin of the London Mathematical Society}, 53\penalty0
  (3):\penalty0 792--797, 2021.

\bibitem[Riemann(1859)]{riemann}
B.~Riemann.
\newblock {Ueber die Anzahl der Primzahlen unter einer gegebenen Grosse}.
\newblock \emph{Gesammelte math. Werke und wissensch}, 2:\penalty0 145–155,
  1859.

\bibitem[Sondow(2003)]{sondow_2003}
J.~Sondow.
\newblock {Zeros of the Alternating Zeta Function on the Line R(s) = 1}.
\newblock \emph{The American Mathematical Monthly}, 110\penalty0 (5):\penalty0
  435--437, 2003.

\bibitem[Titchmarsh(1986)]{Titchmarsh_1986}
E.C. Titchmarsh.
\newblock \emph{The theory of the Riemann Zeta-function (2nd ed.)}.
\newblock The Clarendon Press Oxford University, 1986.

\bibitem[Walz(2017)]{walz_identity}
G.~Walz.
\newblock \emph{{Lexikon der Mathematik}}.
\newblock Springer Spektrum Verlag, 2017.

\bibitem[Widder(1941)]{widder_1941}
D.V. Widder.
\newblock \emph{{Laplace Transform}}.
\newblock Princeton Mathematical Series, 1941.

\end{thebibliography}

\newpage
\appendix


\section{\RH~step by step}\label{app:rh_step_by_step}
There is a strong empirical evidence that \RH~holds. Recent numerical verification by \cite{platt_rh_true} showed that \RH~is  at least true in the region $\{z = a + i b \in \complex: a\in (0,1),b \in (0, \gamma]\}$ where $\gamma = 3\cdot 10^{12}$, meaning that all zeros of the zeta function with imaginary parts in $(0,\gamma]$ have a real part equal to $\frac{1}{2}$. Other theoretical insights seem to support \RH. For instance, french mathematician A. Denjoy gave the following probabilistic argument for \RH~(mentioned in \cite{edwards_1974}): if $(\mu(k))_{k\geq 0}$ is a sequence of independent Bernoulli random variables with values $+1$ or $-1$ (each with probability $1/2$), then for any $\varepsilon>0$, we have $\sum_{k \leq x} \mu(k) = \mathcal{O}_{x \to \infty}(x^{1/2 + \varepsilon})$ with probability $1$. This statement is closely related to \RH~as shown by British mathematician J.E. Littlewood in 1912 (mentioned in \cite{Titchmarsh_1986}). Indeed, Littlewood showed that \RH~is akin to having $\sum_{k\leq x} \mu(k) = \mathcal{O}_{x \to \infty}(x^{1/2+\varepsilon})$ for all $\varepsilon>0$, where $\mu$ is the M\"{o}bius function, a function with values in $\{-1,0,1\}$ that gives the parity of the number of prime factors in the prime decomposition of integers (see e.g. \cite{hardy_1938} for details about the M\"{o}bius function). Hence, with respect to Denjoy's argument, \RH~(informally) suggests that the sequence of M\"{o}bius function $\mu(x)$ behaves like a random walk. This provides a probabilistic interpretation of \RH~but it does not constitute a proof.
There are many consequences to \RH, and probably the most important of these is the distribution of prime numbers. Recent work by \cite{dudek_prime2015} showed that \RH~implies that for any real number $x \geq 2$, there exists a prime number $p_x$ such that $x - \frac{4}{\pi} \sqrt{x} \log(x) \leq p \leq x$. Another consequence of \RH~is with regards to the growth rate of the Mertens function defined by $M(x) = \sum_{k \leq x} \mu(k)$ where $\mu$ is the M\"{o}bius function; \RH~implies that the Mertens function satisfies $M(x) = \mathcal{O}_{x\to\infty}(x^{1 + \varepsilon})$ for all $\varepsilon > 0$. Such functions are ubiquitous in number theory and quantifying their growth rate has several applications. Another intriguing consequence of \RH~is given by the Nyman-Beurling criterion \citep{nyman, beurling} which states that \RH~holds if and only if a special class of functions is dense in $L_2(0,1)$. This class consists of neural networks with one dimensional input, and have a special parameterization.\\

The Riemann hypothesis conjectures that the non-trivial zeros of the Riemann zeta function are complex numbers with a real part $\frac{1}{2}$\footnote{Trivial zeros are negative even numbers, see below for more details.}. It is a long-standing open problem in number theory first formulated by \cite{riemann}. The Riemann zeta function is defined for complex numbers $z$ with a real part greater than 1 by 
\begin{equation}\label{eq:zeta_function}
\zeta(z) = \sum_{n=1}^\infty \frac{1}{n^z}, \quad z \in \mathbb{C}, \realpart(z)>1. 
\end{equation}
The above definition of Riemann zeta function excludes the region of interest $\{z \in \complex: \realpart(z)=\frac{1}{2}\}$ since the series in \cref{eq:zeta_function} diverge when $|z|<1$. Indeed, \RH~is stated for the \emph{analytic continuation} of the zeta function which is an extension of the zeta function on a larger set. As shown by Riemann, the function $\zeta$ extends to the whole complex plane $\complex$ while preserving some desirable properties such as analyticity. This extension is called the analytic continuation, and it is unique by the Identity theorem \citep{walz_identity}. Let us construct this analytic continuation of $\zeta$ step by step. 

\begin{enumerate}

\item \textbf{Extension to $\{ z \in \complex: \realpart(z)>0\}$. } 
Observe that the function $\zeta$ satisfies the following identity
$$
(1 - 2^{1-z})\zeta(z) = \sum_{n=1}^\infty \frac{1}{n^z} - 2 \sum_{n=1}^\infty \frac{1}{(2n)^z} =  \sum_{n=1}^\infty \frac{(-1)^{n+1}}{n^z},
$$
where the right hand side is defined for any complex number $z$ such that $\realpart(z)>0$. The series $\eta(z) = \sum_{n=1}^\infty \frac{(-1)^{n+1}}{n^z}$, known as the Dirichlet eta function, is defined for complex numbers $z$ satisfying $\realpart(z)>0$. Using this expression, we can extend the definition of $\zeta$ to the set $\{z \in \complex: \realpart(z)> 0, 2^{1-z} \neq 1\}$ by
\begin{equation}\label{eq:continuation_zeta}
 \zeta(z) = \frac{1}{1 - 2^{1-z}} \sum_{n=1}^\infty \frac{(-1)^{n+1}}{n^z}.   
\end{equation}

This definition extends the original domain of definition of $\zeta$ to all complex numbers such that $\realpart(z)>0$ except for those satisfying $2^{1-z}=1$, which are all of the form $z_n = 1 + i \frac{2 \pi n}{\log(2)}$, where $n$ is an non-zero integer.  Using classical properties of the Dirichlet eta function $\eta(z) = \sum_{n=1}^\infty \frac{(-1)^{n+1}}{n^z}$ \citep{borwein1995etafunction}, also known as the alternating zeta function, the zeta function $\zeta$ can be analytically continued to include the set $\{z_n , n \geq 1\}$ \citep{widder_1941, sondow_2003}. 

\item \textbf{Extension to $\{z \in \complex: \realpart(z)\leq 0\}$\textbackslash$\{0\}$. }
The Dirichlet function $\eta$ satisfies the following functional equation \citep{borwein1995etafunction}
\begin{equation}\label{eq:eta_functional_eq}
\eta(z) = 2 \frac{2^{z-1} - 1}{2^z - 1} \pi^{z-1} z\, \sin{\frac{\pi z}{2}} \Gamma(-z) \eta(1-z), 
\end{equation}
where $\Gamma$ is the Gamma function. Using \cref{eq:eta_functional_eq}, \cref{eq:continuation_zeta}, and the property of the Gamma function $\Gamma(z+1) = z \Gamma(z)$ for all complex numbers $z\in \complex$ ($\Gamma$ is the analytic continuation of the original Gamma function defined on $\{z\in \complex, \realpart(z)>0\}$ by $\Gamma(z) = \int_0^\infty t^{z-1} e^{-t} dt$), we obtain that for any $z\in \complex$ such that $\realpart(z) \in (0,1)$,
\begin{equation}\label{eq:func_eq1}
\zeta(z) = 2^z \pi^{z-1} \sin\left(\frac{\pi z}{2}\right) \Gamma(1-z) \zeta(1-z).
\end{equation}
Using the functional equation \cref{eq:func_eq}, we can extend the definition of $\zeta$ to the all the remaining complex numbers $z$ such that $\realpart(z) \leq 0$ and $z \neq 0$. It can further be shown that $\zeta$ can be continued to $0$ with $\zeta(0) = -1/2$ \citep{borwein1995etafunction}.

\end{enumerate}

\paragraph{Zeros of the $\zeta$ function.} From \cref{eq:func_eq}, the zeta function satisfies $\zeta(-2k) = 0$ for any integer $k \geq 1$. The negative even integers $\{-2 k \}_{k \geq 1 }$ are thus called \emph{trivial zeros} of the Riemann zeta function since the result follows from the simple fact that $\sin\left( - \pi k \right) = 0$ for all integers $k \geq 1$. The other zeros of $\zeta$ are called non-trivial zeros, and their properties remain largely misunderstood. The \RH~conjectures that they all lie on a the line $\realpart(z)=\frac{1}{2}$.

\paragraph{Riemann Hypothesis (\RH).} \emph{All non-trivial zeros of the Riemann zeta function $\zeta$ have a real part equal to $\frac{1}{2}$}.\\

Whether \RH~holds is still an open question. However, there have been a number of attempts to prove or disprove \RH~ in the literature. In the next section, we re-visit an old result that provides an analytic point of view of \RH.

\section{The Nyman-Beurling Criterion}\label{app:nyman_beurling}

\textbf{Theorem \ref{thm:nyman} }(\cite{nyman}).
\emph{The \RH~is true if and only if $\nn$ is dense in $L_2(I_1)$.}\\

\cite{beurling} later extended this result by showing that for any $p>1$, the $\zeta$ function has no zeroes in the set $\{ z \in \complex: \realpart(z)>1/p\}$ if and only if the set $\nn$ is dense in $L_p(I_1)$.\\

\noindent\textbf{Theorem \ref{thm:beurling}} (\cite{beurling})
\emph{The Riemann zeta function is free from zeros in the half plane $Re(z)>\frac{1}{p}, 1<p<\infty$, if and only if $\nn$ is dense in $L_p(I_1)$.}\\

The following is a sketch of the proof by \cite{beurling}. It helps understand the machinery of the proofs of \cref{thm:nyman} and \cref{thm:beurling} for the sufficient condition. 
\paragraph{Sketch of the proof. } The connection between \RH~ and the class $\nn$ is due to the following identity that relates the zeta function $\zeta$ to the fractional part function $\rho$\footnote{After multiple attempts, we could not find the original paper where this identity has first appeared. However, it has been mentioned in different works, e.g. \citep{nyman, beurling}.}
\begin{equation}\label{eq:zeta_rho}
\int_{0}^1 \rho\left(\frac{\theta}{x}\right) x^{z-1} dx = \frac{\theta}{z-1} - \frac{\theta^z \zeta(z)}{z}, \quad \forall z \in \complex, \realpart(z)>0, \theta \in I.
\end{equation}

Therefore, given a function $f \in \nn$, we obtain
\begin{equation}\label{eq:zeta_f}
\int_{0}^1 f(x) x^{z-1} dx = - \frac{ \zeta(z) \sum_{i=1}^m c_i \beta_i^{z-1}}{z}, \quad \forall z \in \complex, \realpart(z)>0.
\end{equation}
Now fix $p>1$ and assume that the class $\nn$ is dense in $L_p(I)$. Therefore, given $\varepsilon >0$, there exists a function $f \in \nn$ such that $\| \mathbf{1} - f\| < \varepsilon$, where $\mathbf{1}$ is the constant function on $I$ equal to $1$ everywhere. Using \cref{eq:zeta_f}, we obtain
$$
\int_{0}^1 (1 - f(x)) x^{z-1} dx = \frac{1}{z}\left(1 + \zeta(z) \sum_{i=1}^m c_i \beta_i^{z-1}\right), \quad \forall z \in \complex, \realpart(z)>0.
$$
Using H\"{o}lder's inequality, we have for any $z$ such that $\realpart(z)>1/p$
$$
\|\mathbf{1} - f\|_p \| x^{z-1}\|_q  < \varepsilon \frac{1}{(q(\realpart(z) - 1/p))^{1/q}},
$$
where $q>1$ satisfies $1/p + 1/q = 1$. This yields,
$$
|1 - \zeta(z) \sum_{i=1}^m c_i \beta_i^z|^q < \frac{\varepsilon^q |z|^q}{ q ( \realpart(z) - 1/p)}.
$$

Hence, in the region $\{ z \in \complex, \realpart(z)>1/p + \varepsilon^q |z|^q / q\}$,  $\zeta(z)$ cannot be equal to zero. Since $\varepsilon$ is arbitrarily chosen, we conclude that $\zeta(z)\neq 0$ if $\realpart(z)>1/p$.\\
For the necessary condition, We invite the reader to check \cite{beurling}. 

The identity \eqref{eq:zeta_rho} is the nub of the proof above. It provides an integral representation of the zeta function $\zeta$ in terms of the fractional part function $\rho$. Hence, one would expect that some properties of zeta function should in-principle be reflected on some function classes involving the function $\rho$. This is precisely the idea behind Theorems \ref{thm:nyman} and \ref{thm:beurling}. 
More importantly, from the analysis above, we have the following analytic criterion for zero-free regions of the zeta function.

\textbf{Lemma \ref{lemma:zero_free_region_1} }(Nyman-Beurling zero-free regions)
\emph{Let $f \in \nn$ and $\delta = \|1 - f\|_2$ be the distance between the constant function $1$ on $I$ and $f$. Then, the region $\{z \in \complex, \realpart(z) > \frac{1}{2} \left(1 + \delta^2 |z|^2\right)\}$ is free of zeroes of the Riemann zeta function $\zeta$.}

Existing empirical verification studies of the \RH~use different analytic criteria to locate the zeroes. To the best of our knowledge, the most recent verification study was conducted by \cite{platt_rh_true} where the authors have found that in the region $\{a + i b: a\in (0,1), b \in (0, \gamma], \gamma \approx 3 \cdot 10^{12}\}$, the \RH~is satisfied, meaning that all the zeroes of $\zeta$ are on the line $\realpart(z)=1/2$. Is it possible to use Lemma \ref{lemma:zero_free_region_1} to beat this record? first, notice that unless the function $f$ is simple, the norm $\|1 - f\|_2$ is generally intractable, an can only be approximated with Monte-Carlo sampling. See \cref{app:empirical_implications} for more details. 

The result of Lemma \ref{lemma:zero_free_region_1} (and that of \cref{thm:beurling}) is stated for the function class $\nn$ which consists of a special neural architecture with one-dimensional inputs. \emph{Can we generalize this to multi-dimensional inputs?} In the next section, we answer this question positively by introducing a generalized class of neural networks with multi-dimensional inputs.

\section{A sufficient condition in the multi-dimensional case}\label{app:sufficient_condition}

Let $d \geq 1$ and consider the following class of neural networks with inputs in $I_d$,

$$
\nn_d = \{ f(x) = \sum_{j=1}^d \sum_{i=1}^m c_{i,j} \rho\left(\frac{\beta_{i,j}}{x_j}\right), x \in I^d: m \geq 1, c \in \reals^{d\times m}, \beta \in I_{d \times m}, c^T\beta = 0 \},
$$
where $c = (c_{1,1}, c_{2,1}, \dots, c_{m,1}, \dots, c_{m, d}) \in \reals^{m\times d}$ is the flattened vector of $(c_{.,j})_{1 \leq j \leq d}$. Notice that we recover the Nyman-Beurling class when $d=1$. Using this class, we can generalize the zero-free region result given by \cref{lemma:zero_free_region_1} to a multi-dimensional setting.\\

\noindent\textbf{Lemma \ref{lemma:zero_free_region_d} }[zero-free regions for general $d \geq 1$]\\
\emph{Let $d\geq 1$ and $f \in \nn_d$. Let $\delta = \|1 - f\|_2$ be the distance between the constant function $1$ on $I_d$ and $f$. Then, the region $\{z \in \complex, \realpart(z) > \frac{1}{2} \left(1 + \delta^{\frac{2}{d}} |z|^2\right)\}$ is free of zeroes of the Riemann zeta function $\zeta$.}

\begin{proof}
Let $d\geq 1$ and $f \in \nn_d$. Using \cref{eq:zeta_rho}, namely  $\int_{0}^1 \rho\left(\frac{\alpha}{x}\right) x^{z-1} dx = \frac{\alpha}{z-1} - \frac{\alpha^z \zeta(z)}{z}$ for $\alpha \in (0,1), \realpart(z)>1/2$, we have that 
\begin{align*}
\int_{I_d} f(x) \prod_{j=1}^d x_j^{z-1} dx &= z^{-d + 1 } \sum_{j=1}^d \sum_{i=1}^{m} c_{i,j} \left( \frac{\beta_{i,j}}{z-1} - \frac{\beta_{i,j}^z \zeta(z)}{z}\right) \\
&= - z^{-d} \sum_{j=1}^d \sum_{i=1}^{m} c_{i,j} \beta_{i,j}^z \zeta(z),
\end{align*}
where we have used the fact that $c^T \beta = 0$.

Therefore, we have
\begin{equation}\label{eq:identity_necessary}
\int_{I_d} (1 - f(x)) \prod_{j=1}^d x_j^{z-1} dx = z^{-d} (1  + \zeta(z) \sum_{j=1}^d \sum_{i=1}^{m} c_{i,j} \beta_{i,j}^z).    
\end{equation}

Using Cauchy-Schwartz inequality, we obtain 

$$
\left|1   + \zeta(z) \sum_{j=1}^d \sum_{i=1}^{m} c_{i,j} \beta_{i,j}^z\right|^2 \leq \|1- f\|^2 \frac{|z|^{2d}}{(2 \realpart(z) - 1)^d},
$$
where we have used the fact that $\left\|\prod_{j=1}^d x_j^{z-1}\right\|^2_2 = (2 \realpart(z) - 1)^{-d}$. Therefore, for all complex numbers $z$ satisfying $\realpart(z) > \frac{1}{2} + \frac{\delta^{2/d} |z|^2}{2}$, we have $\zeta(z) \neq 0$. This is true since we can choose $f$ such that $\|1 - f\|$ is arbitrarily close to $\delta$.
\end{proof}
Notice that if $\delta$ can be chosen arbitrarily small, then the zero-free region in \cref{lemma:zero_free_region_d} can be extended to the whole half-plane $\{\realpart(z) > 1/2\}$. This is a generalization of the sufficient condition of \cref{thm:beurling} in the multi-dimensional case.

\end{document}